\documentclass{article}
\usepackage[dvips]{graphicx}
\usepackage{amssymb,amsmath,color}
\usepackage{url}

\title{Structured sparsity-inducing norms \\ through submodular functions}

\author{
Francis Bach  \\
INRIA - Willow project-team\\
Laboratoire d'Informatique de l'Ecole Normale Sup\'erieure \\
Paris, France \\
\texttt{francis.bach@ens.fr} }

\newcommand{\BEAS}{\begin{eqnarray*}}
\newcommand{\EEAS}{\end{eqnarray*}}
\newcommand{\BEA}{\begin{eqnarray}}
\newcommand{\EEA}{\end{eqnarray}}
\newcommand{\BEQ}{\begin{equation}}
\newcommand{\EEQ}{\end{equation}}
\newcommand{\BIT}{\begin{itemize}}
\newcommand{\EIT}{\end{itemize}}
\newcommand{\BNUM}{\begin{enumerate}}
\newcommand{\ENUM}{\end{enumerate}}
\newcommand{\BA}{\begin{array}}
\newcommand{\EA}{\end{array}}

\newcommand{\tr}{\mathop{ \rm tr}}

\newcommand{\idm}{I}
\newcommand{\rb}{\mathbb{R}}
\newcommand{\BlackBox}{\rule{1.5ex}{1.5ex}}  
\newcommand{\lova}{Lov\'asz }

\newenvironment{proof}{\par\noindent{\bf Proof\ }}{\hfill\BlackBox\\[2mm]}
\newtheorem{lemma}{Lemma}

\newtheorem{proposition}{Proposition}

\newcommand{\mysec}[1]{Section~\ref{sec:#1}}
\newcommand{\eq}[1]{Eq.~(\ref{eq:#1})}
\newcommand{\myfig}[1]{Figure~\ref{fig:#1}}

\def \supp{ { \rm Supp }}

\begin{document}

\maketitle

\begin{abstract}
Sparse methods for supervised learning aim at finding good linear predictors from as few variables as possible, i.e., with small cardinality of their supports. This combinatorial selection problem is often turned into a convex optimization problem by replacing the cardinality function by its convex envelope (tightest convex lower bound), in this case the $\ell_1$-norm. In this paper, we investigate more general set-functions than the cardinality, that may incorporate prior knowledge or structural constraints which are common in many applications: namely, we show that for nondecreasing submodular set-functions, the corresponding convex envelope can be obtained from its \lova extension, a common tool in submodular analysis. This defines a family of polyhedral norms, for which we provide generic algorithmic tools (subgradients and proximal operators) and theoretical results (conditions for support recovery or high-dimensional inference). By selecting specific submodular functions,  we can give a new interpretation to known norms, such as those based on rank-statistics or grouped norms with potentially overlapping groups; we also define new norms, in particular ones that can be used as non-factorial priors for supervised learning.\end{abstract}

\section{Introduction}

The concept of parsimony is central in many scientific domains. In the context of statistics, signal processing or machine learning, it takes the form of variable or feature selection problems, and is commonly used in two situations: First,  to make the model or the prediction more interpretable or cheaper to use, i.e., even if
the underlying problem does not admit sparse solutions, one looks for the best sparse approximation. Second, sparsity  can also be used given prior knowledge that the model should be sparse. In these two situations, reducing parsimony to finding models with low cardinality turns out to be limiting, and structured parsimony has emerged as a fruitful practical extension, with applications to image processing, text processing or bioinformatics (see, e.g.,~\cite{cap,jenatton2009structured,huang2009learning,LaurentGuillaumeGroupLasso,kim,jenattonmairal,Mairal10aNIPS} and \mysec{examples}). For example, in~\cite{LaurentGuillaumeGroupLasso}, structured sparsity is used to encode prior knowledge regarding network relationship between genes, while in~\cite{jenattonmairal}, it is used as an alternative to structured non-parametric Bayesian process based priors for topic models.

Most of the work based on convex optimization and the design of dedicated sparsity-inducing norms has focused mainly on the specific allowed set of sparsity patterns~\cite{cap,jenatton2009structured,LaurentGuillaumeGroupLasso,jenattonmairal}: if $w \in \rb^p$ denotes the predictor we aim to estimate, and $\supp(w)$ denotes its support, then these norms are designed so that penalizing with these norms only leads to supports from a  given family of allowed patterns. In this paper, we instead follow the approach of~\cite{haupt2006signal,huang2009learning} and consider specific penalty functions $F(\supp(w))$ of the support set, which go beyond the cardinality function, but are not limited or designed to only forbid certain sparsity patterns. As shown in \mysec{patterns}, these may also lead to restricted sets of supports but their interpretation in terms of an \emph{explicit} penalty on the support leads to additional insights into the behavior of structured sparsity-inducing norms (see, e.g., \mysec{overlapping}). While direct greedy approaches (i.e., forward selection) to the problem are considered in \cite{haupt2006signal,huang2009learning}, we provide convex relaxations to the function $w \mapsto  F(\supp(w))$, which extend the traditional link between the $\ell_1$-norm and the cardinality function.

This is done for a particular ensemble of set-functions $F$, namely \emph{nondecreasing submodular functions}. Submodular functions may be seen as the set-function equivalent of convex functions, and exhibit many interesting properties that we review in \mysec{submodular}---see~\cite{submodular_tutorial} for a tutorial on submodular analysis and \cite{krause2005near,kawahara22submodularity} for other applications to machine learning.  This paper makes the following contributions:

\hspace*{.25cm} $-$   We make explicit links between submodularity and sparsity by showing that the convex envelope of the function $w \mapsto  F(\supp(w))$ on the $\ell_\infty$-ball may be readily obtained from  the \lova extension of the submodular function (\mysec{norm}). 

\hspace*{.25cm} $-$ We provide generic algorithmic tools, i.e., subgradients and proximal operators (\mysec{prox}), as well as theoretical guarantees, i.e., conditions for support recovery or high-dimensional inference (\mysec{analysis}), that extend classical results for the $\ell_1$-norm and show that many norms may be tackled by the exact same analysis and algorithms.

\hspace*{.25cm} $-$  By selecting specific submodular functions in \mysec{examples},  we recover and give a new interpretation to known norms, such as those based on rank-statistics or grouped norms with potentially overlapping groups~\cite{cap,jenatton2009structured,Mairal10aNIPS}, and we define new norms, in particular ones that can be used as non-factorial priors for supervised learning (\mysec{examples}). These are illustrated on simulation experiments in \mysec{simulations}, where they outperform related greedy approaches~\cite{huang2009learning}.

\textbf{Notation.} \hspace*{.15cm}
For $w \in \rb^p$, $\supp(w) \subset V = \{1,\dots,p\}$ denotes the support of $w$, defined as $\supp(w) = \{ j \in V, \ w_j \neq 0 \}$. 
For $w \in \rb^p$ and $q \in [1,\infty]$, we  denote by $\| w\|_q$ the $\ell_q$-norm
of $w$. We denote by $|w| \in \rb^p$ the vector of absolute values of the components of $w$.
Moreover, given a vector $w$ and a matrix~$Q$, $w_A$ and $Q_{AA}$ are the corresponding subvector and submatrix of $w$ and $Q$. Finally, for $w \in \rb^p$ and $A \subset V$, $w(A)=\sum_{k \in A} w_k$ (this defines a modular set-function).

\section{Review of submodular function theory}

\label{sec:submodular}
\label{sec:review}
Throughout this paper, we consider a \emph{nondecreasing submodular} function $F$ defined on the power set $2^V$ of $V = \{1,\dots,p\}$, i.e., such that:
 \BEAS
 \forall A,B \subset V, & &   F(A) + F(B) \geqslant F(A \cup B) + F(A \cap B), \hspace*{.2cm} \mbox{ (submodularity)} \\
 \forall A,B \subset V, & &   A \subset B \Rightarrow F(A) \leqslant F(B). \hspace*{2.547cm} \mbox{ (monotonicity)}
  \EEAS
 Moreover, we assume (without loss of generality) that $F(\varnothing)=0$. These set-functions are often referred to as \emph{polymatroid set-functions}~\cite{fujishige2005submodular} or \emph{$\beta$-functions}~\cite{edmonds}. Also, without loss of generality, we may assume that $F$ is strictly positive on singletons, i.e., for all $k\in V$, $F(\{ k\})>0$. Indeed, if $F(\{k\})=0$, then by submodularity and monotonicity, if $A \ni k$, $F(A) = F(A \backslash \{k\})$ and thus we can simply consider $V \backslash \{k\}$ instead of $V$.

Classical examples are the cardinality function (which will lead to the $\ell_1$-norm) and, given a partition of $V$ into $B_1 \cup \cdots \cup B_k=V$, the set function $A\mapsto F(A)$ which is equal to 
the number of groups $B_1,\dots,B_k$ with non empty intersection with $A$ (which will lead to the grouped $\ell_1$/$\ell_\infty$-norm~\cite{cap, negahban2008joint}).
 
 \textbf{\lova extension.} \hspace*{.15cm}
 Given any set-function $F$, one can define its \emph{\lova extension}~\cite{lovasz1982submodular} (a.k.a.~\emph{Choquet integral}~\cite{choquet1953theory}) $f: \rb_+^p \to \rb$, as follows: given $w \in \rb_+^p$, we can order the components of $w$ in decreasing order $w_{j_1} \geqslant \dots \geqslant w_{j_p} \geqslant 0$; the value $f(w)$ is then defined as:
 \BEQ
 f(w) 
\label{eq:lovasz2}
   =    \sum_{k=1}^{p} w_{j_k} [ F( \{ j_1,\dots,j_k\} ) - F( \{ j_1,\dots,j_{k-1}\} ) ] .
 \EEQ
 Note that if some of the components of $w$ are equal, all orderings lead to the same value of $f(w)$.
 The  \lova extension $f$ is always piecewise-linear, and when $F$ is submodular, it is also convex (see, e.g.,~\cite{lovasz1982submodular,fujishige2005submodular}). Moreover,  for all $\delta \in \{0,1\}^p$, $f(\delta) = F(\supp(\delta))$: $f$ is indeed an extension from  vectors in $\{0,1\}^p$ (which can be identified with indicator vectors of sets) to all vectors in $\rb_+^p$. Moreover,  it turns out that minimizing $F$ over subsets, i.e., minimizing $f$ over $\{0,1\}^p$ is equivalent to minimizing $f$ over $[0,1]^p$~\cite{lovasz1982submodular,edmonds}.
 
\textbf{Submodular polyhedron and greedy algorithm.}  \hspace*{.15cm}
We denote by $\mathcal{P}$ the \emph{submodular polyhedron}~\cite{fujishige2005submodular}, defined as  the set of $s \in \rb_+^p$ 
such that for all $A \subset V$, $s(A) \leqslant F(A)$, i.e., 
$
\mathcal{P} = \{ s \in \rb_+^p, \ \forall A \subset V, \  s(A)\leqslant F(A) \},
$
where we use the notation $s(A) = \sum_{k \in A} s_k$. One important result in submodular analysis is that if $F$ is a nondecreasing submodular function, then we have a representation of $f$ as a maximum of linear functions~\cite{fujishige2005submodular,lovasz1982submodular}, i.e.,  for all $w \in \rb_+^p$, 
\BEQ
\label{eq:poly}
f(w) = \max_{ s \in \mathcal{P}} \  w^\top s.
\EEQ
Instead of solving a linear program with $p+2^p$ contraints, a solution $s$ may be obtained by the following ``greedy algorithm'':
order the components of $w$ in decreasing order $w_{j_1} \geqslant \dots \geqslant w_{j_p}$, and then take for all $k \in   \{1,\dots,p\}$,
$s_{j_k} = F( \{ j_1,\dots,j_k\} ) - F( \{ j_1,\dots,j_{k-1}\} ) .$

\textbf{Stable sets.}\hspace*{.15cm}
 A set $A$ is said \emph{stable} if it cannot be augmented without increasing $F$, i.e.,  if  for all sets $B \supset A$, $B \neq A \Rightarrow F(B) > F(A)$. If $F$ is strictly increasing, then all sets are stable. Stable sets are also sometimes referred to as \emph{flat} or \emph{closed}~\cite{edmonds}. The set   of stable sets is closed by intersection~\cite{edmonds}, and will correspond to the set of allowed sparsity patterns (see \mysec{patterns}). For the cardinality function, all sets are stable.

\textbf{Separable sets.}\hspace*{.15cm}
 A set $A$ is separable if we can find a partition of $A$ into $ A = B_1 \cup \cdots \cup B_k$ such that $F(A) = F(B_1)+\cdots+F(B_k)$. A set $A$ is inseparable if it is not separable. As shown in~\cite{edmonds}, the submodular polytope $\mathcal{P}$ has full dimension $p$ as soon as $F$ is strictly positive on all singletons, and its faces are exactly the sets $\{ s_k=0 \}$ for $k \in V$ and $\{ s(A) = F(A) \}$ for stable \emph{and} inseparable sets. We let denote $\mathcal{T}$ the set of such sets. This implies that $\mathcal{P} = \{ s \in \rb_+^p, \ \forall A \in \mathcal{T}, s(A) \leqslant F(A) \}$. These stable inseparable sets will play a role when describing extreme points of unit balls of our new norms (\mysec{norm}) and for deriving concentration inequalities in \mysec{highdim}. For the cardinality function, stable and inseparable sets are singletons.

\paragraph{Submodular function minimization.}
Submodular functions are particularly interesting because they can be minimized in polynomial time. In this paragraph, we consider a non-monotonic submodular function $G$ (otherwise finding the minimum is trivial). Most algorithms for minimizing submodular functions rely on the following strong duality principle~\cite{edmonds,fujishige2005submodular}:
\BEQ
\label{eq:mini}
\min_{ A \subset V} G(A) = \max_{s \in \mathcal{B}(G)}  \sum_{k \in V} \min \{0,s_k \},
\EEQ
where $\mathcal{B}(G) = \{ s \in \rb^p, \forall A \subset V, s(A) \leqslant G(A), s(V) = G(V) \}$ is referred to as the \emph{base polyhedron}. Moreover, algorithms for minimizing $G$ will usually output $A$  and $s$ such that $G(A) = \sum_{k \in V} \min \{ 0,s_k\}$ as a certificate for optimality. The two main types of algorithms are combinatorial algorithms (that explicitly looks for $A$) and ones based on convex optimization (that explicitly looks for $s$). The first type of algorithm leads to strongly polynomial algorithms with best known complexity $O(p^6)$~\cite{orlin2009faster}, while the minimum point algorithm of~\cite{fujishige2005submodular} has no worst-time complexity bounds but is usually much faster in practice~\cite{fujishige2005submodular} and is based on the equivalent problem of finding the minimum-norm point in $\mathcal{B}(G)$, i.e., $\min_{s \in \mathcal{B}(G)  }   \| s\|_2^2
$. Note that in this case, the minimum point algorithm also outputs a particular $s$ solution of    \eq{mini}---which has several solutions in general.

\section{Definition and properties of structured norms}

\label{sec:norm}
We define the function $\Omega(w) = f(|w|)$, where $|w|$ is the vector in $\rb^p$ composed of absolute values of~$w$ and $f$ the \lova extension of $F$. We have the following properties (see proof in the appendix), which show that we indeed define a norm and that it is the desired convex envelope:
\begin{proposition}[Convex envelope, dual norm] 
\label{prop:envelope}
Assume that the set-function $F$ is submodular, nondecreasing, and strictly positive for all singletons. Define  
$\Omega: w \mapsto f(|w|)$. Then:

(i) $\Omega$ is a norm on $\rb^p$,

(ii) $\Omega$ is the convex envelope of the function $g: w \mapsto F( \supp(w))$ on the unit $\ell_\infty$-ball,

(iii) the dual norm (see, e.g.,~\cite{boyd}) of $\Omega $ is equal to $  \Omega^\ast(s) = \max_{A \subset V } \frac{ \| s_A \|_1 }{F(A)} 
 = \max_{A \in \mathcal{T} } \frac{ \| s_A \|_1 }{F(A)} $.
\end{proposition}

\begin{figure}
\begin{center}

\hspace*{-.5cm}
\includegraphics[scale=.32]{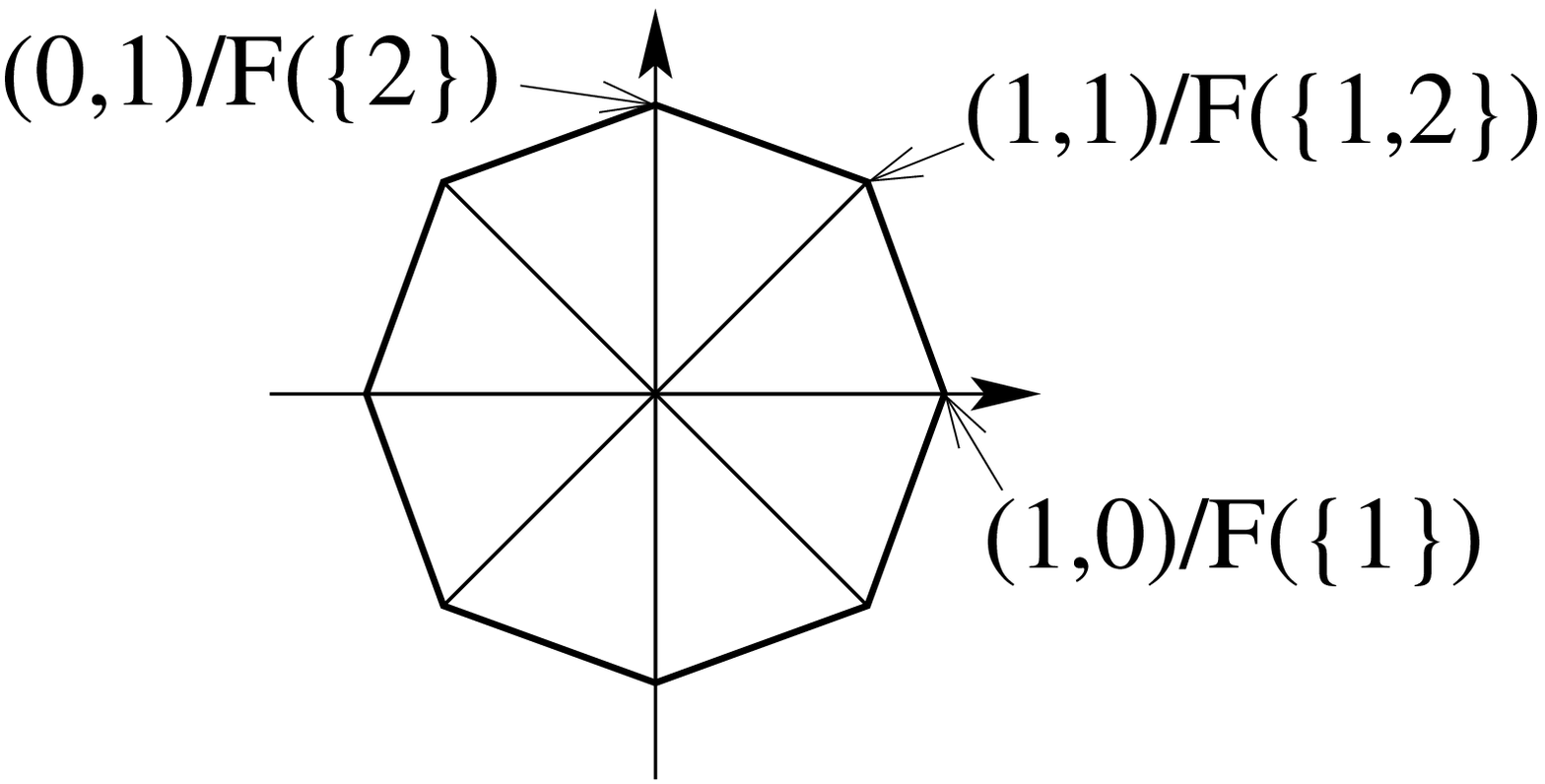} 
\hspace*{.5cm}
\includegraphics[scale=.32]{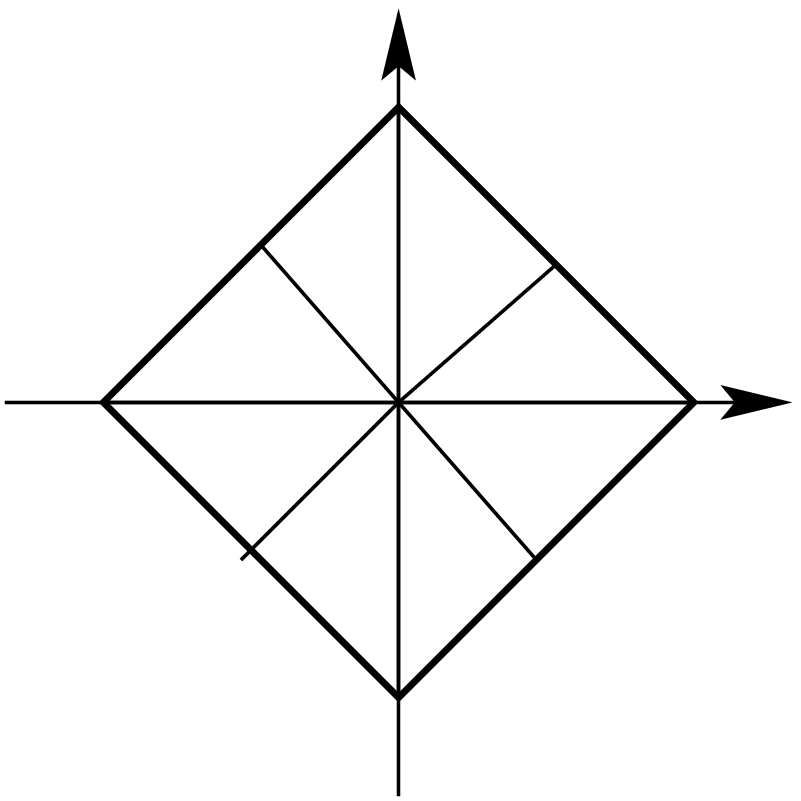}
\hspace*{-.05cm}
\includegraphics[scale=.32]{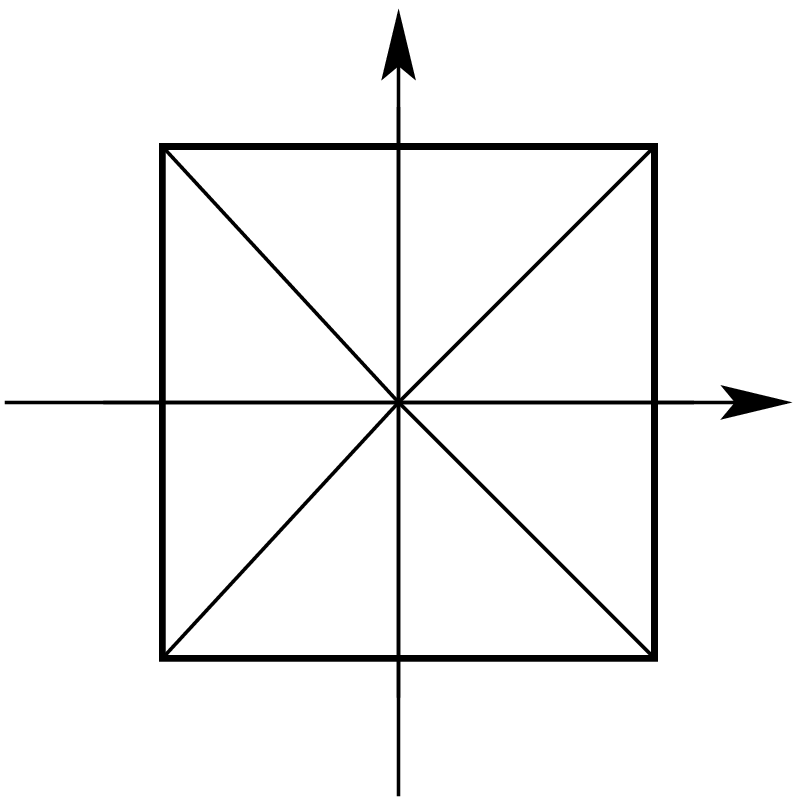}
\hspace*{-.05cm}
\includegraphics[scale=.32]{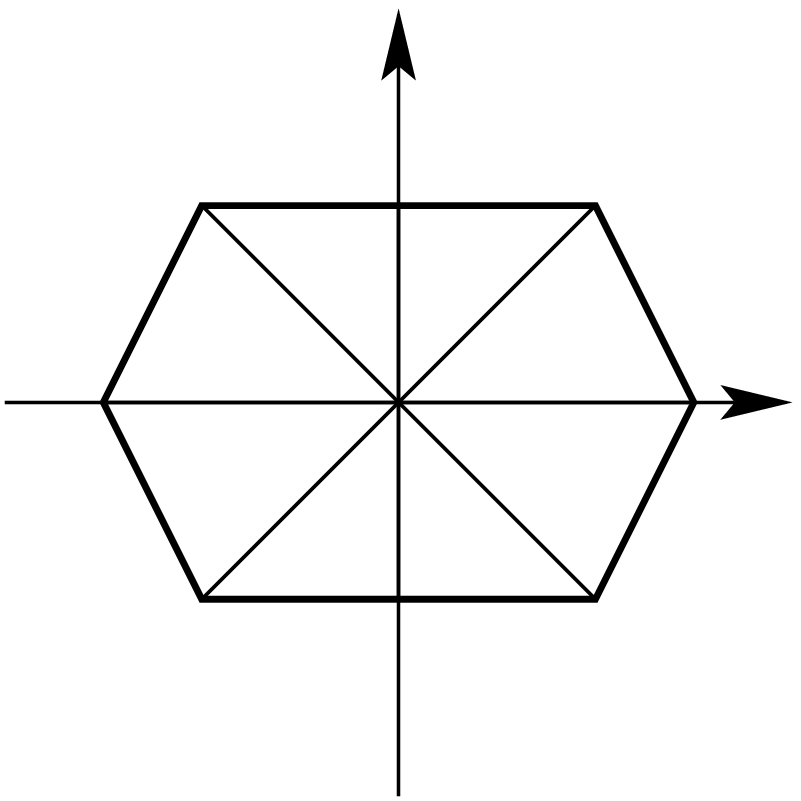} 
\hspace*{-.05cm}

\end{center}
\caption{Polyhedral unit ball, for 4 different submodular functions (two variables), with different stable inseparable sets leading to different sets of extreme points; changing values of $F$ may make some of the extreme points disappear. From left to right: $F(A) = |A|^{1/2}$ (all possible extreme points), $F(A) = |A|$ (leading to the $\ell_1$-norm),
$F(A) = \min\{|A|,1\}$ (leading to the $\ell_\infty$-norm), $F(A) =  \frac{1}{2} 1_{ \{A \cap \{2\} \neq \varnothing\} } +  1_{ \{A \neq \varnothing\} }  $
(leading to the  structured norm $\Omega(w) = \frac{1}{2}|w_2 | +   \| w\|_\infty$).
}
\label{fig:balls}
\end{figure}

We provide examples of submodular set-functions and norms in \mysec{examples}, where we go from set-functions to norms, and vice-versa.
From the definition of the \lova extension in \eq{lovasz2}, we see that $\Omega$ is a polyhedral norm (i.e., its unit ball is a polyhedron). The following proposition gives the set of extreme points of the unit ball (see proof in the appendix and examples in  \myfig{balls}):
\begin{proposition}[Extreme points of unit ball]
\label{prop:extremepoints}
The extreme points of the unit ball of $\ \Omega$ are the vectors $\frac{1}{F(A)}s$, with $s \in \{-1,0,1\}^p$, $\supp(s) = A$ and $A$ a stable inseparable set.
\end{proposition}
This proposition shows, that depending on the number and cardinality of the inseparable stable sets, we can go from $2p$ (only singletons) to $3^p-1$ extreme points (all possible sign vectors). We show in \myfig{balls} examples of balls for $p=2$, as well as  sets of extreme points. These extreme points will play a role in concentration inequalities derived in \mysec{analysis}.

\section{Examples of nondecreasing submodular functions}

\label{sec:examples}

We consider three main types of submodular functions with potential applications to regularization for supervised learning. Some existing norms are shown to be examples of our frameworks (\mysec{overlap}, \mysec{card}), while other novel norms are designed from specific submodular functions (\mysec{eig}).
Other examples of submodular functions, in particular in terms of matroids and entropies, may be found in~\cite{fujishige2005submodular,krause2005near,kawahara22submodularity} and could also lead to interesting new norms. Note that set covers, which are  common examples of submodular functions are subcases of set-functions defined in \mysec{overlap} (see, e.g.,~\cite{submodular_tutorial}).

\subsection{Norms defined with non-overlapping or overlapping groups}
\label{sec:overlap}
\label{sec:overlapping}

We consider grouped norms defined with potentially overlapping groups~\cite{cap,jenatton2009structured}, i.e.,  $\Omega(w) = \sum_{G \subset V} d(G) \| w_{G} \|_\infty$ where $d$ is a nonnegative set-function (with potentially $d(G)=0$ when $G$ should not be considered in the norm).
It is a norm as soon as   $\cup_{G, d(G) >0 } G = V$ and it corresponds to the nondecreasing submodular function $F(A) = \sum_{G \cap A \neq \varnothing}d(G)$. In the case where  $\ell_\infty$-norms are replaced by $\ell_2$-norms, \cite{jenatton2009structured} has shown that the set of allowed sparsity patterns are intersections of complements of groups~$G$ with strictly positive weights. These sets happen to be the set of stable sets for the corresponding submodular function; thus the analysis provided in \mysec{patterns} extends the result of  \cite{jenatton2009structured} to the new case of
  $\ell_\infty$-norms. 
  However, in our situation, we can give a reinterpretation through a submodular function that counts the number of times the support $A$ intersects groups~ $G$ with non zero weights. This goes beyond restricting the set of allowed sparsity patterns to stable sets. We show later in this section some insights gained by this reinterpretation.
We now give some examples of norms, with various topologies of groups.

 \textbf{Hierarchical norms.} \hspace*{.15cm}
Hierarchical norms defined on directed acyclic graphs~\cite{cap,kim,jenattonmairal} correspond to the set-function $F(A)$ which is the cardinality of the union of ancestors of elements in $A$. These have been applied to bioinformatics~\cite{kim}, computer vision and topic models~\cite{jenattonmairal}.

 \begin{figure}
  \begin{center}

\includegraphics[scale=.3]{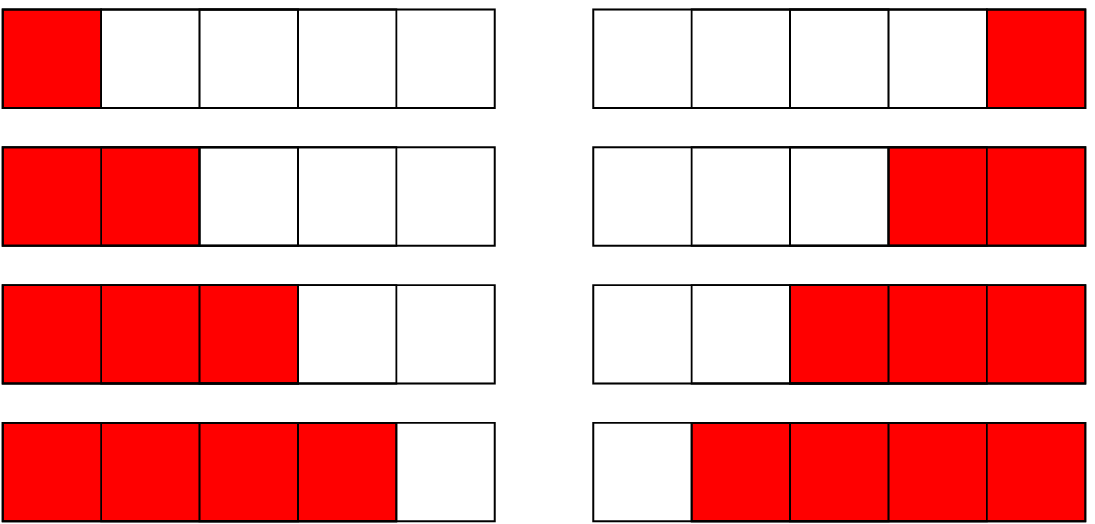} \hspace*{2cm}
\includegraphics[scale=.3]{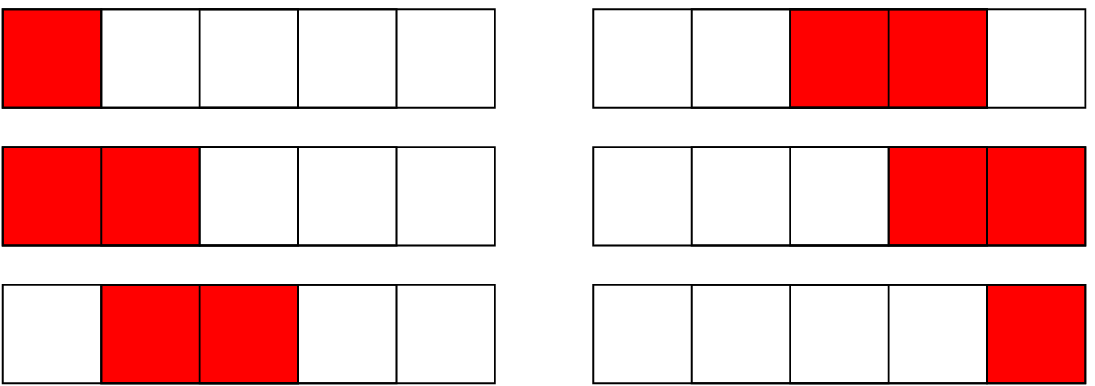}

  \caption{Sequence and groups: (left) groups for contiguous patterns, (right) groups for penalizing the number of jumps in the indicator vector sequence.}
  \label{fig:sequence}
 \end{center}

\end{figure}
 
 \begin{figure}
  \begin{center}

  \hspace*{-.2cm}
\includegraphics[width=3.2cm]{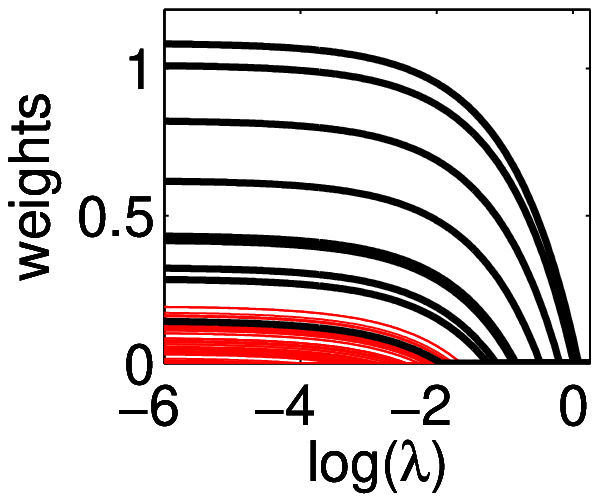} \hspace*{.52cm}
\includegraphics[width=3.2cm]{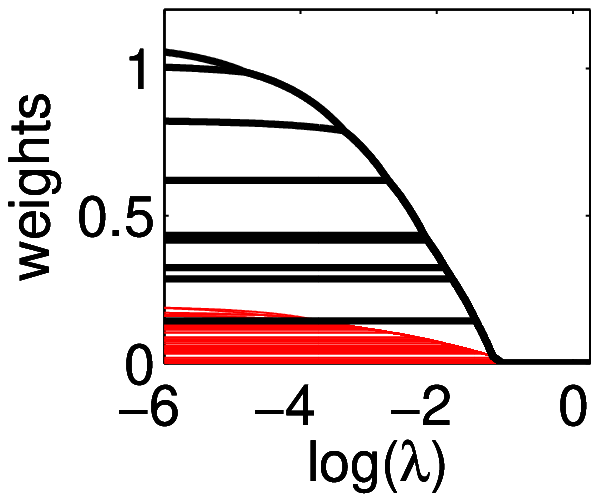}\hspace*{-.05cm}
\includegraphics[width=3.2cm]{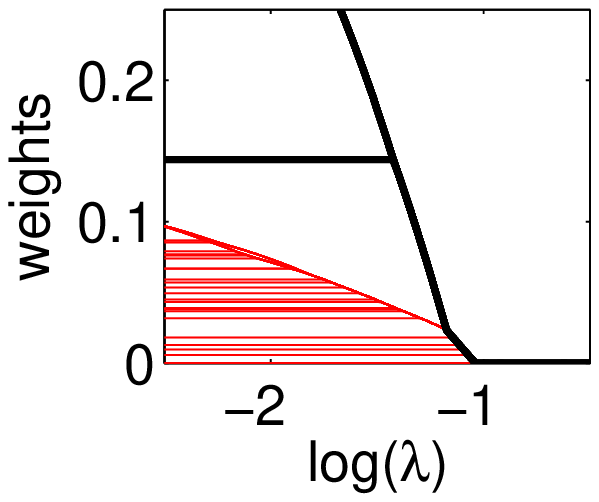}\hspace*{.52cm}
\includegraphics[width=3.2cm]{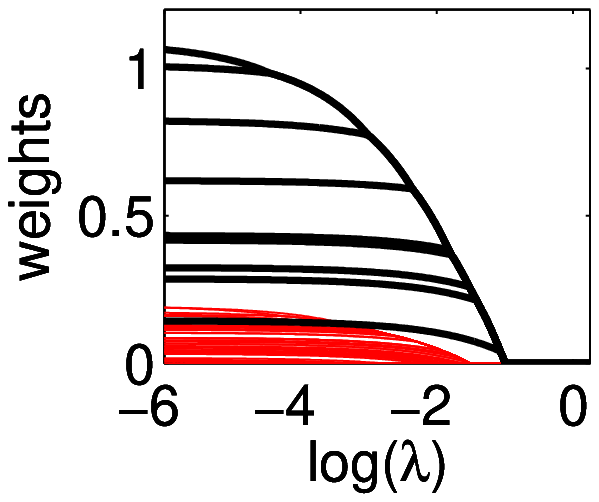}
\hspace*{-.2cm}

  \caption{Regularization path for a penalized least-squares problem (black: variables that should be active, red: variables that should be left out). From left to right: $\ell_1$-norm penalization (a wrong variable is included with the correct ones), polyhedral norm for rectangles in 2D, with zoom (all variables come in together), mix of the two norms (correct behavior).}
  \label{fig:mix}
 \end{center}

\end{figure}

\textbf{Norms defined on grids.} \hspace*{.15cm}
If we assume that the $p$ variables are organized in a 1D, 2D or 3D grid,~\cite{jenatton2009structured} considers norms based on overlapping groups leading to stable sets equal to rectangular or convex shapes, with applications in computer vision~\cite{SparseStructuredPCA}. For example, for the groups defined in the left side of \myfig{sequence} (with unit weights), we have $F(A)=p-2+{\rm range}(A)$ if $A \neq \varnothing$ and $F(\varnothing) = 0$  (the range of $A$ is equal to $\max(A) - \min(A)+1$). From empty sets to non-empty sets, there is a gap of $p-1$, which is larger than differences among non-empty sets. This leads to the undesired result, which has been already observed by~\cite{jenatton2009structured}, of adding all variables in one step, rather than gradually, when the regularization parameter decreases in a regularized optimization problem. In order to counterbalance this effect, adding a constant times the cardinality function has the effect of making the first gap relatively smaller. This corresponds to adding a constant times the $\ell_1$-norm and, as shown in \myfig{mix}, solves the problem of having all variables coming together. All patterns are then allowed, but contiguous ones are \emph{encouraged rather than forced}.

Another interesting new norm may be defined from the groups in the right side of \myfig{sequence}. Indeed, it corresponds to the function $F(A)$ equal to $|A|$ plus the number of intervals of $A$. Note that this also favors contiguous patterns but is  not limited to selecting a single interval (like the norm obtained from groups  in the left side of \myfig{sequence}). Note that it is to be contrasted with the total variation (a.k.a.~fused Lasso penalty~\cite{tibshirani2005sparsity}), which is a relaxation of the number of jumps in a vector $w$ rather than in its support. In 2D or 3D, this extends to the notion of perimeter and area, but we do not pursue such extensions here.

\subsection{Spectral functions of submatrices}
\label{sec:eig}

Given a positive semidefinite matrix $Q \in \rb^{ p \times p }$ and a real-valued  function $h$ from $\rb_+ \to \rb$, one may define $\tr [h(Q)]$ as $\sum_{i=1}^p h(\lambda_i)$ where $\lambda_1,\dots,\lambda_p$ are the (nonnegative) eigenvalues of $Q$~\cite{horn1990matrix}. We can thus define the set-function $F(A) = \tr h(Q_{AA})$ for $A \subset V$.
The functions $h(\lambda) = \log( \lambda + t)$ for $t\geqslant 0$ lead to submodular functions, as they correspond to  entropies of Gaussian random variables (see, e.g.,~\cite{fujishige2005submodular,submodular_tutorial}). Thus, since for $q \in (0,1)$,
$\lambda^q = \frac{ q \sin q \pi }{\pi} \int_0^\infty  \log (1+ \lambda/t) t^{q-1} dt$~(see, e.g.,~\cite{ando1979concavity}), $h(\lambda) = \lambda^q$ for $q \in (0,1]$ are positive linear combinations of functions that lead to nondecreasing submodular functions. Thus, they are also nondecreasing submodular functions, and, to the best of our knowledge, provide novel examples of such functions.

In the context of supervised learning from a design matrix $X \in \rb^{n \times p}$, we naturally use $Q=X^\top X$.
If $h$ is linear, then $F(A) = \tr X_A^\top X_A = 
\sum_{k \in A} X_k^\top X_k $ (where $X_A$ denotes the submatrix of $X$ with columns in $A$) and we obtain a weighted cardinality function and hence and a weighted $\ell_1$-norm, which is a \emph{factorial prior}, i.e., it is a sum of terms depending on each variable independently.

In a frequentist setting, the Mallows $C_L$ penalty~\cite{mallows} depends on the degrees of freedom, of the form $\tr  X_{A}^\top  X_{A} ( X_{A}^\top  X_{A}+  \lambda\idm)^{-1}$. This is a non-factorial prior but unfortunately it does not lead to a submodular function. In a Bayesian context however, it is shown by~\cite{wipf} that penalties of the form 
$\log \det (  X_{A}^\top  X_{A}+ \lambda\idm) $ (which lead to submodular functions) correspond to marginal likelihoods associated to the set $A$ and have good behavior when used within a non-convex framework.
This highlights the need for non-factorial priors which are sub-linear functions of the eigenvalues of $X_A^\top X_A$, which is exactly what nondecreasing submodular function of submatrices are.
 We do not pursue the extensive evaluation of non-factorial convex priors in this paper but provide in simulations examples with $F(A) = \tr (X_{A}^\top X_{A})^{1/2}$ (which is   equal to the trace norm of $X_A$~\cite{boyd}).

\subsection{Functions of cardinality}
\label{sec:card}

For $F(A) = h(|A|)$ where $h$ is nondecreasing, such that $h(0)=0$ and concave, then, from \eq{lovasz2}, $\Omega(w)$ is defined from the rank statistics of $|w| \in \rb_{+}^p$, i.e., if $|w_{(1)} | \geqslant
| w_{(2)}|  \geqslant \cdots \geqslant |w_{(p)}|
 $, then
 $\Omega(w) = \sum_{k=1}^p  [ h(k)-h(k-1) ] |w_{(k)}|$.  This includes the sum of the $q$ largest elements, and might lead to interesting new norms for unstructured variable selection but this is not pursued here. However, the algorithms and analysis presented in \mysec{prox} and \mysec{analysis} apply to this case.

\section{Convex analysis and optimization}

\label{sec:prox}
 \label{sec:optimization}
 In this section we provide  algorithmic tools related to  optimization problems based on the regularization by our novel sparsity-inducing norms. Note that since these norms are polyhedral norms with unit balls having potentially  an exponential number of vertices or faces, regular linear programming toolboxes may not be used.
 
\textbf{Subgradient.} \hspace*{.15cm}
From $\Omega(w) = \max_{ s \in \mathcal{P}} s^\top |w|$ and the greedy algorithm\footnote{The greedy algorithm to find extreme points of the submodular polyhedron should not be confused with the greedy algorithm (e.g., forward selection) that we consider in \mysec{simulations}.} presented in \mysec{submodular}, one can easily get in \emph{polynomial time} one subgradient as one of the maximizers $s$. This allows to use subgradient descent, with, as shown in \myfig{runningtimes}, slow convergence compared to proximal methods.

\textbf{Proximal operator.} \hspace*{.15cm}
Given regularized problems of the form $\min_{w \in \rb^p} L(w) + \lambda \Omega(w)$, where $L$ is differentiable with Lipschitz-continuous gradient, \emph{proximal methods} have been shown to be particularly efficient first-order methods~(see, e.g.,~\cite{beck2009fast}). In this paper, we consider the methods ``ISTA'' and its accelerated variants ``FISTA''~\cite{beck2009fast}, which are compared in \myfig{runningtimes}.

To apply these methods, it suffices to be able to
solve efficiently problems of the form: $\min_{ w \in \rb^p} \frac{1}{2} \| w - z \|_2^2 + \lambda \Omega(w)$. In the case of the $\ell_1$-norm, this reduces to soft thresholding of~$z$, the following proposition (see proof in the appendix) shows that this is equivalent to a particular algorithm for submodular function minimization, namely the minimum-norm-point algorithm, which has no complexity bound but is empirically faster than algorithms with such bounds~\cite{fujishige2005submodular}:

\begin{proposition}[Proximal operator]
\label{prop:proximal}
Let $z \in \rb^p$ and $\lambda>0$, minimizing  $\frac{1}{2} \| w - z \|_2^2 + \lambda \Omega(w)$ is equivalent to finding the minimum of the submodular function $A \mapsto \lambda F(A) -  |z|(A)$ with the minimum-norm-point algorithm.
\end{proposition}

In the proof, it is shown how a solution for one problem may be obtained from a solution to the other problem. Moreover, any algorithm for minimizing submodular functions allows to get directly the support of the unique solution of the proximal problem and that with a sequence of submodular function minimizations, the full solution may also be obtained. Similar links between convex optimization and minimization of submodular functions have been considered (see, e.g., \cite{chambolle2009total}). However, these are dedicated to \emph{symmetric} submodular functions (such as the ones obtained from graph cuts) and are thus not directly applicable to our situation of \emph{non-increasing} submodular functions.

Finally, note that using the minimum-norm-point algorithm leads to a \emph{generic} algorithm that can be applied to \emph{any} submodular functions $F$, and that it may be rather inefficient for simpler subcases (e.g., the $\ell_1/\ell_\infty$-norm, tree-structured groups~\cite{jenattonmairal}, or general overlapping groups~\cite{Mairal10aNIPS}).

\section{Sparsity-inducing properties}

\label{sec:sparsity}
\label{sec:analysis}
In this section, we consider a fixed design matrix $X \in \rb^{n \times p}$  and $y \in \rb^n$ a vector of random responses. Given $\lambda >0$, we define
$\hat{w}$ as a minimizer of the regularized least-squares cost:
\BEQ
\label{eq:objective}
\textstyle
\min_{w \in \rb^p} \textstyle \frac{1}{2n} \| y - X w\|_2^2 + \lambda \Omega(w).
\EEQ
We study the sparsity-inducing properties of solutions of \eq{objective}, i.e., we determine in \mysec{patterns} which patterns are allowed and in \mysec{high} which sufficient conditions lead to correct estimation. Like recent analysis of sparsity-inducing norms~\cite{negahban2009unified}, the analysis provided in this section relies heavily on decomposability properties of our norm $\Omega$.

\subsection{Decomposability}

For a subset $J$ of $V$, we denote
by $F_J: 2^J \to \rb$ the \emph{restriction} of $F$ to $J$, defined for $A \subset J$ by $F_J(A) = F(A)$, and by $F^J: 2^{J^c} \to \rb$ the \emph{contraction} of $F$ by $J$, defined for $A \subset J^c$ by $F^J(A) = F(A \cup J) - F(A)$.  These two functions are submodular and nondecreasing as soon as $F$ is (see, e.g.,~\cite{fujishige2005submodular}).

We denote by $\Omega_{J}$ the norm on $\rb^J$ defined through the submodular function $F_J$, and $\Omega^J$ the pseudo-norm defined on $\rb^{J^c}$ defined through $F^J$ (as shown in Proposition~\ref{prop:decomposability}, it is a norm only when $J$ is a stable set). Note that $\Omega_{J^c}$ (a norm on $J^c$) is in general different from $\Omega^J$. Moreover, $\Omega_J(w_J)$ is actually equal to $\Omega(\tilde{w})$ where $\tilde{w}_J = w_J$ and $\tilde{w}_{J^c} = 0$, i.e., it is the restriction of $\Omega$ to $J$.

We can now prove the following decomposition properties, which show that under certain circumstances, we can decompose the norm $\Omega$ on subsets $J$ and their complements:

\begin{proposition}[Decomposition]
\label{prop:decomposability}
Given $J \subset V$ and $\Omega_J $ and $\Omega^J $ defined as above, we have:

(i) $\forall w \in \rb^p$, $\Omega(w) \geqslant \Omega_J(w_J) +  \Omega^{J}(w_{J^c})$,

(ii) $\forall w \in \rb^p$,  if $\min_{ j \in J } | w_j| \geqslant \max_{ j \in J^c } |w_j|$ , then $\Omega(w) = \Omega_J(w_J) +  \Omega^{J}(w_{J^c})$,


(iii) $\Omega^J$ is a norm on $\rb^{J^c}$ if and only if $J$ is a stable set.
\end{proposition}

\subsection{Sparsity patterns}
\label{sec:patterns}

In this section, we do not make any assumptions regarding the correct specification of the linear model. We show that with probability one, only stable support sets may be obtained  (see proof in the appendix). For simplicity, we assume invertibility of $X^\top X$, which forbids the high-dimensional situation $p\geqslant n$ we consider in \mysec{high}, but we could consider assumptions similar to the ones used in~\cite{jenatton2009structured}.  

\begin{proposition}[Stable sparsity patterns]
\label{prop:patterns}
Assume  $y \in \rb^n$  has an absolutely continuous density with respect to
the Lebesgue measure and that $X^\top X $ is invertible. Then the minimizer $\hat{w}$ of \eq{objective} is unique and, 
 with probability one, its support
$ \supp(\hat{w})$ is a stable set.
\end{proposition}

\subsection{High-dimensional inference}
\label{sec:high}

We now assume that the linear model is well-specified and extend results from~\cite{Zhaoyu} for sufficient support recovery conditions and from~\cite{negahban2009unified} for estimation consistency. As seen in Proposition~\ref{prop:decomposability}, the norm $\Omega$ is decomposable and we use this property extensively in this section.
We denote by $\rho(J) = \min_{B \subset J^c} \frac{ F(B \cup J) - F(J) } { F(B)}$; by submodularity and monotonicity of $F$, $\rho(J)$ is always between zero and one, and,  as soon as $J$ is stable it is strictly positive (for the $\ell_1$-norm, $\rho(J)=1$). Moreover, we denote by $c(J) = \sup_{w \in \rb^p }   { \Omega_J(w_J) }/ { \| w_J \|_2 }$, the equivalence constant between the norm $\Omega_J$ and the $\ell_2$-norm. We always have $c(J) \leqslant |J|^{1/2} \max_{k \in V} F(\{k\}) $ (with equality for the $\ell_1$-norm).

The following propositions allow us to get back and extend well-known results for the $\ell_1$-norm, i.e.,  Propositions~\ref{prop:support} and~\ref{prop:proba} extend results based on support recovery conditions~\cite{Zhaoyu};
while Propositions \ref{prop:high-dim} and \ref{prop:proba} extend results based on restricted eigenvalue conditions~(see, e.g.,~\cite{negahban2009unified}). We can also get back results for the $\ell_1$/$\ell_\infty$-norm~\cite{negahban2008joint}. As shown in the appendix,   proof techniques are similar and are adapted through the  decomposition properties from Proposition~\ref{prop:decomposability}.

\begin{proposition}[Support recovery]
\label{prop:support}
Assume that $y = Xw^\ast + \sigma \varepsilon$, where $\varepsilon$ is a standard multivariate normal vector. Let $Q = \frac{1}{n} X^\top X \in \rb^{p \times p}$.  Denote by $J$ the smallest stable set containing the  support $\supp(w^\ast)$ of $w^\ast$. Define $\nu = \min_{ j, w^\ast_j \neq 0 } | w^\ast_j | >0$, assume $\kappa = \lambda_{\min} (Q_{JJ}) > 0$ and that for $\eta >0$,
$(\Omega^J)^\ast[  ( \Omega_J(   Q_{JJ}^{-1} Q_{Jj} ) )_{j \in J^c} ] \leqslant 1 - \eta$.
 Then, if $\lambda \leqslant \frac{\kappa \nu }{2 c(J)} $, the minimizer $\hat{w}$ is unique and has support equal to $J$, with   probability larger than 
 $
1-3 P \big( \Omega^\ast(z) > \frac{ \lambda \eta \rho(J)  \sqrt{n} }{ 2 \sigma   } \big)
 $,
 where $z$ is a multivariate normal with covariance matrix $Q$.
 
 \end{proposition}
 \label{sec:highdim}

\begin{proposition}[Consistency]
\label{prop:high-dim}
Assume that $y = Xw^\ast + \sigma \varepsilon$, where $\varepsilon$ is a standard multivariate normal vector. Let $Q = \frac{1}{n} X^\top X \in \rb^{p \times p}$. Denote by $J$ the smallest stable set containing the  support $\supp(w^\ast)$ of $w^\ast$. 
Assume that for all $\Delta$ such that   $\Omega^J(\Delta_{J^c}) \leqslant 3 \Omega_J(\Delta_J)$, $\Delta^\top Q \Delta \geqslant \kappa \|\Delta_J \|_2^2$. Then we have $  \Omega(\hat{w} - w^\ast) \leqslant  \frac{24 c(J)^2 \lambda}{\kappa \rho(J)^2}  \mbox{ and } \frac{1}{n} \| X \hat{w} - X w^\ast\|_2^2 \leqslant 
  \frac{36 c(J)^2 \lambda^2}{\kappa \rho(J)^2}$,
 with probability larger than 
$ 1 - P \big( \Omega^\ast(z) > \frac{ \lambda \rho(J) \sqrt{n}  }{ 2 \sigma  } \big) $
where $z$ is a multivariate normal with covariance matrix $Q$.
   \end{proposition}

   \begin{proposition}[Concentration inequalities]
  \label{prop:proba}
 Let $z$ be a normal variable with covariance matrix $Q$. Let $\mathcal{T}$ be the set of stable inseparable sets. Then
 $ \textstyle P( \Omega^\ast(z) > t )  
 \leqslant  \sum_{A \in \mathcal{T}} 2^{|A|}  \exp \big( - \frac{   t^2 F(A)^2 /2 }{  1^\top Q_{AA} 1} \big).$
  \end{proposition}

\section{Experiments}
\label{sec:simulations}

We provide illustrations on toy examples of some of the results presented in the paper. We consider the regularized least-squares problem of \eq{objective}, with data generated as follows: given $p,n,k$, the design matrix $X \in \rb^{n \times p}$ is a matrix of i.i.d.~Gaussian components, normalized to have unit $\ell_2$-norm columns. A set $J$ of cardinality $k$ is chosen at random and the weights $w^\ast_J$ are sampled from a standard multivariate Gaussian distribution and $w^\ast_{J^c}=0$. We then take $y = Xw^\ast+ n^{-1/2} \| Xw^\ast\|_2  \  \varepsilon$ where $\varepsilon$ is a standard Gaussian vector (this corresponds to a unit signal-to-noise ratio).

\textbf{Proximal methods vs.~subgradient descent.} \hspace*{.15cm}
For the submodular function $F(A) = |A|^{1/2}$ (a simple submodular function beyond the cardinality) we compare three optimization algorithms described in \mysec{optimization}, subgradient descent and two proximal methods, ISTA and its accelerated version FISTA~\cite{beck2009fast}, for $p=n=1000$, $k=100$ and $\lambda=0.1$. Other settings and other set-functions would lead to similar results than the ones presented in \myfig{runningtimes}: FISTA is faster than ISTA, and much faster than subgradient descent.

\begin{figure}
  \begin{center}

  \hspace*{-.5cm}
\includegraphics[width=4.3cm]{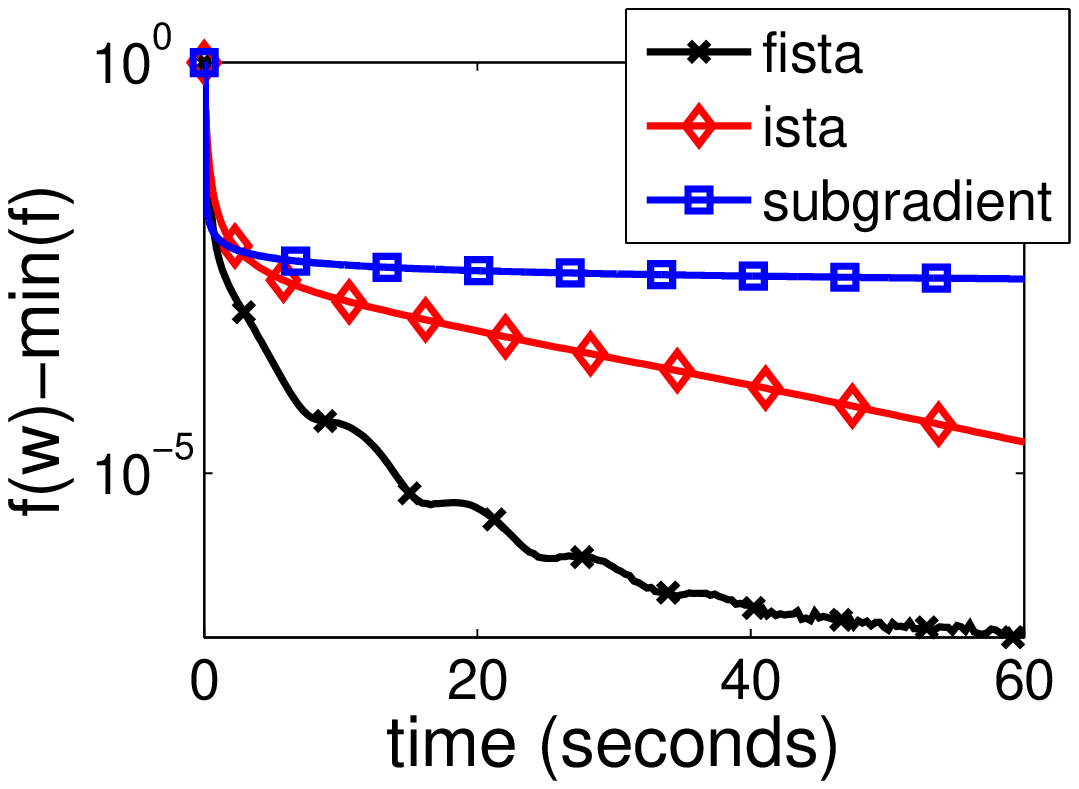} \hspace*{.5cm}
\includegraphics[width=4.3cm]{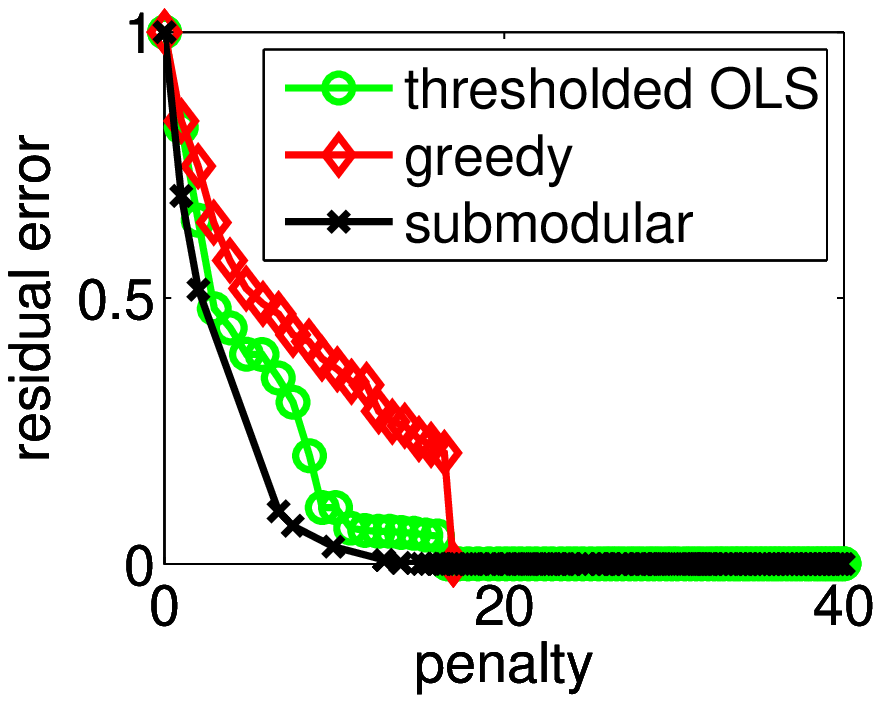} \hspace*{.15cm}
\includegraphics[width=4.3cm]{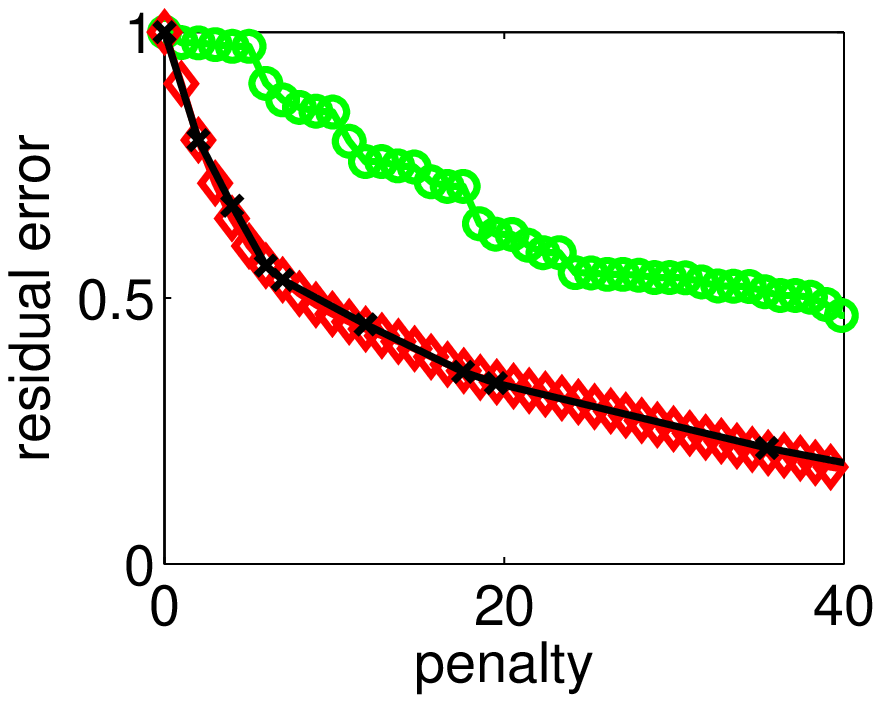}
 \hspace*{-.5cm}
\

  \caption{(Left) Comparison of iterative optimization algorithms (value of objective function vs.~running time). 
(Middle/Right)  Relaxation of combinatorial optimization problem, showing residual error $\frac{1}{n} \| y - X\hat{w} \|_2^2$ vs. penalty $F(\supp(\hat{w}))$: (middle) high-dimensional case ($p=120$, $n=20$, $k=40$), 
(right) lower-dimensional case ($p=120$, $n=120$, $k=40$). }
  \label{fig:runningtimes}
 \end{center}

\end{figure}

\textbf{Relaxation of combinatorial optimization problem.} \hspace*{.15cm}
We compare three strategies for solving the combinatorial optimization problem $\min_{w \in \rb^p} \frac{1}{2n} \|y  - X w\|_2^2 + \lambda F(\supp(w))$ with $F(A) = \tr (X_A^\top X_A)^{1/2}$, the approach based on our sparsity-inducing norms, the  simpler greedy (forward selection) approach proposed in \cite{haupt2006signal,huang2009learning}, and by thresholding the ordinary least-squares estimate. For all methods, we try all possible regularization parameters.
We see in the right plots of \myfig{runningtimes} that for hard cases (middle plot) convex optimization techniques perform better than other approaches, while for easier cases with more observations (right plot), it does as well as greedy approaches.

\textbf{Non factorial priors for variable selection.} \hspace*{.15cm}
We  now focus on the predictive performance and compare our new norm  with $F(A) = \tr (X_A^\top X_A)^{1/2}$, with greedy approaches~\cite{huang2009learning} and to regularization by $\ell_1$ or $\ell_2$ norms. As shown in Table~\ref{tab:perf}, the new norm based on non-factorial priors is more robust than the $\ell_1$-norm to lower number of observations $n$ and to larger cardinality of support $k$.

\begin{table}

\begin{center}
\begin{tabular}{|lll|r|rrr|}
\hline
$p$ & $n$ & $k$ & submodular & $\ell_2$ vs. submod. & $\ell_1$ vs. submod.   & greedy vs. submod.  \\
\hline
  120 & 120 & 80 & 40.8 $\pm$ 0.8 & -2.6  $\pm$ 0.5 & \bf 0.6  $\pm$ 0.0  & \bf 21.8  $\pm$ 0.9  \\ 
120 & 120 & 40 & 35.9 $\pm$ 0.8 & \bf 2.4  $\pm$ 0.4 & \bf 0.3  $\pm$ 0.0  & \bf 15.8  $\pm$ 1.0  \\ 
120 & 120 & 20 & 29.0 $\pm$ 1.0 & \bf 9.4  $\pm$ 0.5 & -0.1  $\pm$ 0.0  & \bf 6.7  $\pm$ 0.9  \\ 
120 & 120 & 10 & 20.4 $\pm$ 1.0 & \bf 17.5  $\pm$ 0.5 & -0.2  $\pm$ 0.0  & -2.8  $\pm$ 0.8  \\ 
120 & 120 & 6 & 15.4 $\pm$ 0.9 & \bf 22.7  $\pm$ 0.5 & -0.2  $\pm$ 0.0  & -5.3  $\pm$ 0.8  \\ 
120 & 120 & 4 & 11.7 $\pm$ 0.9 & \bf 26.3  $\pm$ 0.5 & -0.1  $\pm$ 0.0  & -6.0  $\pm$ 0.8  \\ 
 \hline
 120 & 20 & 80 & 46.8 $\pm$ 2.1 & -0.6  $\pm$ 0.5 & \bf 3.0  $\pm$ 0.9  & \bf 22.9  $\pm$ 2.3  \\ 
120 & 20 & 40 & 47.9 $\pm$ 1.9 & -0.3  $\pm$ 0.5 & \bf 3.5  $\pm$ 0.9  & \bf 23.7  $\pm$ 2.0  \\ 
120 & 20 & 20 & 49.4 $\pm$ 2.0 & 0.4  $\pm$ 0.5 & \bf 2.2  $\pm$ 0.8  & \bf 23.5  $\pm$ 2.1  \\ 
120 & 20 & 10 & 49.2 $\pm$ 2.0 & 0.0  $\pm$ 0.6 & 1.0  $\pm$ 0.8  & \bf 20.3  $\pm$ 2.6  \\ 
120 & 20 & 6 & 43.5 $\pm$ 2.0 & \bf 3.5  $\pm$ 0.8 & \bf 0.9  $\pm$ 0.6  & \bf 24.4  $\pm$ 3.0  \\ 
120 & 20 & 4 & 41.0 $\pm$ 2.1 & \bf 4.8  $\pm$ 0.7 & -1.3  $\pm$ 0.5  & \bf 25.1  $\pm$ 3.5  \\ 
 \hline
\end{tabular}

\caption{Normalized mean-square prediction errors $ \| X \hat{w} - X w^\ast\|_2^2/n$ (multiplied by 100) with optimal regularization parameters (averaged over 50 replications, with standard deviations divided by $\sqrt{50}$). The performance of the submodular method is shown, then differences from all methods to this particular one are computed, and shown in bold when they are significantly greater than zero, as measured by a paired t-test with level 5\% (i.e., when the submodular method is significantly better).}
\label{tab:perf}

\end{center}

\end{table}

\section{Conclusions}

We have presented a family of sparsity-inducing norms dedicated to incorporating prior knowledge or structural constraints on the support of linear predictors. We have provided a set of common algorithms and theoretical results, as well as simulations on synthetic examples illustrating the good behavior of these norms. Several avenues are worth investigating: first, we could follow current practice in sparse methods, e.g., by considering related adapted concave penalties to enhance sparsity-inducing norms, or by extending some of the concepts for norms of matrices, with potential applications in matrix factorization or multi-task learning  (see, e.g.,~\cite{cevher} for application of submodular functions to dictionary learning). Second, links between submodularity and sparsity could be studied further, in particular by considering submodular relaxations of other combinatorial functions, or studying links with other polyhedral norms such as the total variation, which are known to be similarly associated with  symmetric submodular set-functions such as graph cuts~\cite{chambolle2009total}.

\textbf{Acknowledgements.} \hspace*{.15cm}
 This paper was partially supported by   the Agence Nationale de la Recherche (MGA Project) and the European Research Council (SIERRA Project). The author would like to thank Edouard Grave, Rodolphe Jenatton, Armand Joulin, Julien Mairal and Guillaume Obozinski for discussions related to this work.

\newpage

\appendix

\section{Properties of the norm}

\subsection{Proof of Proposition \ref{prop:envelope}}
(i) 
$\Omega$ is positively homogeneous by definition of the \lova extension in \eq{lovasz2}, convex because of the representation in \eq{poly} as the maximum of $s^\top|w|$ for some $s \in \mathcal{P} \subset \rb_+^p$, and it is a norm as soon as $\Omega(w)=0$ implies that $w=0$, which is true since $\Omega(w) \geqslant \min_{k} F(\{k\}) \| w\|_\infty$. 
(ii)
We denote by $g^\ast$ the Fenchel conjugate of $g$ on the domain $\{ w \in \rb^p, \ \| w\|_\infty \leqslant 1\}$, and $g^{\ast \ast}$ its bidual~\cite{boyd}. By definition of the Fenchel conjugate, we have:
\BEAS
 g^\ast(s)  
& = & \max_{ \|w\|_\infty \leqslant 1 } w^\top s - g(w) \\
& = & \max_{\delta \in \{0,1\}^p }\max_{ \| w\|_\infty \leqslant 1 }  ( \delta \circ w ) ^\top s - f(\delta) \\
& = & \max_{\delta \in \{0,1\}^p }     \delta  ^\top |s| - f(\delta) \\
& = & \max_{\delta \in [0,1]^p }     \delta  ^\top |s| - f(\delta) \mbox{ because $F - |s|$ is submodular}. \\
\EEAS
Thus, for all $w$ such that $\| w\|_\infty \leqslant 1$, 
\BEAS
 g^{\ast \ast}(w)  & = & \max_{s \in \rb^p} s^\top w - g^\ast(s)  \\
& = & \max_{s \in \rb^p}  \min_{\delta \in [0,1]^p } \  s^\top w    -  \delta  ^\top |s| + f(\delta) \\
& = &  \min_{\delta \in [0,1]^p } \max_{s \in \rb^p}   \  s^\top w    -  \delta  ^\top |s| + f(\delta) \mbox{ by strong duality and Slater's condition~\cite{boyd}} \\
& = &  \min_{\delta \in [0,1]^p, \delta \geqslant | w|  }   f(\delta)  = f(|w|)  \mbox{ because $F$ is nonincreasing}.
\EEAS
Note that $F$ non-increasing implies that $f$ is non-increasing with respect to all of its components.
(ii)
We have
$ \displaystyle
\Omega(w) = f(|w|) =  \max_{ s \in \mathcal{P} } s^\top |w| = \max_{ |s| \in  \mathcal{P} } s^\top w
= \max_{ \| s_A\|_1 \leqslant F(A), \ A \subset V }  s^\top w
= \max_{ \max_{A \subset V} \frac{ \| s_A \|_1 }{F(A)} \leqslant 1} s^\top w$, which implies the desired
result. Note that the maximization may  indeed  be limited to the stable inseparable sets $A \in \mathcal{T}$.

\subsection{Proof of Proposition~\ref{prop:extremepoints}}
 We have seen in \mysec{review} that for $A \in \mathcal{T}$ (set of stable inseparable sets), then $\{ x(A) = F(A) \}$ is a face of $\mathcal{P}$ (and those sets are the only ones for which this happens). We get to the desired result by considering potential different signs.

\section{Convex optimization results}

We first prove an additional result related to decomposition of subdifferentials.
Note that the exact subdifferential for the non-zero components of $w$ is rather complicated when $w$ has components with equal magnitude. If this is not the case, i.e., $|w_{j_1}| > \cdots > |w_{j_k}| > 0$, where $k = |J|$, then the subdifferential  $\partial \Omega_J(w_J)$ is reduced to a point $s$ such that
$s_{j_k}  = F( \{ j_1,\dots,j_k\} ) - F( \{ j_1,\dots,j_{k-1}\} )$. For more details on the subdifferential for nonzero components, see \cite{fujishige2005submodular}.

\begin{lemma}[Decomposition of subdifferential]
\label{prop:subdifferential}
Let $w \in \rb^p$, with support $J = \supp(w)$ and with $H$ equal to the smallest stable set containing $J$. The subdifferential $\partial \Omega(w)$ at $w$, can  then be decomposed as follows on $\rb^V = \rb^J \times \rb^{ H \backslash J} \times \rb^{H^c}$: $\partial \Omega(w) = \partial \Omega_J(w_J) \times \{ 0\} \times \{ s_{H^c}, \ (\Omega^H)^\ast(s_{H^c}) \leqslant 1 \}.$
\end{lemma}
\begin{proof}
For all sufficient small $\Delta \in \rb^p$, the components in $(w+\Delta)_J$ have all greater absolute values than the ones
in $(w+\Delta)_{J^c}$. Thus, from Proposition~\ref{prop:decomposability}, $\Omega(w+\Delta) = \Omega_J(w_J + \Delta_J) + 
\Omega^J(\Delta_{J^c}) = \Omega_J(w_J + \Delta_J) + 
\Omega^H(\Delta_{H^c}) $, and thus the subdifferential decomposes as 
$\partial \Omega_J(w_J) \times \{ 0\} \times \partial \Omega^H(0)$.
The subdifferential of a norm at zero is exactly the unit ball
of the dual norm, which leads to the desired result.
\end{proof}

\subsection{Proof of Proposition~\ref{prop:proximal}}

 Following~\cite{jenattonmairal}, without loss of generality, we assume that $z$ has nonnegative components. We have by convex duality (which is applicable here because of Slater's condition):
\BEAS
\min_{w \in \rb^p} \frac{1}{2} \| w - z \|_2^2 + \lambda \Omega(w)
& = & \min_{w \in \rb^p}  \max_{\Omega^\ast(s) \leqslant 1}   \frac{1}{2} \| w - z \|_2^2 + \lambda s^\top w \\
& = & \max_{\Omega^\ast(s) \leqslant 1}   \min_{w \in \rb^p}    \frac{1}{2} \| w - z \|_2^2 + \lambda s^\top w \\
& = & \max_{\Omega^\ast(s) \leqslant 1}      \frac{1}{2} \|  z \|_2^2 -  \frac{1}{2} \| \lambda s - z \|_2^2,
\EEAS
where the (unique) optimal $w$ is obtained from the optimal $s$ by $w = z - \lambda s$.
$s$ is defined constrained to satisfy $\Omega^\ast(s) \leqslant 1$, which is equivalent to $|s| \in \mathcal{P}$.
 Since $z$ has nonnegative components, the minimum restricted to $|s| \in \mathcal{P}$ is the same as the minimum restricted to $s \in \mathcal{P}$, and also the same as the one restricted to the submodular polyhedron without constraints on positivity, i.e., our problem reduces to
$\min_{ \forall A \subset V, s(A) \subset F(A) }      \|   s - z/\lambda \|_2^2$, which is also equivalent to
\BEQ
\label{eq:AA}
\min_{ \forall A \subset V, t(A) \subset F(A) - \lambda^{-1} z(A)  }      \|   t \|_2^2.
\EEQ
Up to the constraints $s(V) =   F(V) - \lambda^{-1} z(V) $, this is the minimum-norm point problem for the submodular function $G:A \mapsto F(A) - \lambda^{-1} z(A)$. We can then follow two approaches: the first one is to apply directly the minimum-norm point algorithm to the problem in \eq{AA}, which we have followed in simulations. The second approach is to consider the regular minimum point algorithm; we can then follow \cite[Lemma 7.4]{fujishige2005submodular}: if $t$ is the minimum-norm solution for the submodular function $G$, then we can obtain $s$ as $\lambda^{-1} z$ plus the negative part of $t$. From $s$ we then get $w$ through $w = z - \lambda s$.

 If another algorithm is used for submodular function minimization, then, following \cite[Lemma 7.4]{fujishige2005submodular}, we know which components of the (unique) optimal value $t^\ast$ are negative and which of them are equal to zero (which corresponds to zero components of $w^\ast$). Then, following~\cite{chambolle2009total}, if we add a constant vector with components equal to $\alpha$ to $z$, we may obtain level sets of $w^\ast$. With several values of $\alpha$, we can then obtain the full solution $w^\ast$. However, the minimum norm point algorithm remains the most efficient and allows to obtain directly the solution of the proximal problem.

\section{Sparse estimation}
In this section, we consider a design $X \in \rb^{n \times p}$  be a fixed design and $y \in \rb^n$ a set of random responses. Given $\lambda >0$, we define
$\hat{w}$ as a minimizer of the regularized least-squares cost:
\BEQ
\label{eq:objective}
\min_{w \in \rb^p} \frac{1}{2n} \| y - X w\|_2^2 + \lambda \Omega(w).
\EEQ

\subsection{Proof of Proposition~\ref{prop:decomposability}}
 
(i) for $s \in \rb_+^p$, if $\forall B \subset J$, $s(B) \leqslant F(B)$ and $\forall C \subset J^c$, 
$s(C) \leqslant F(C \cap J) - F(J)$, then $\forall A \subset V$,
$s(A) = s(A\cap J) + s(A\cap J^c)
\leqslant F(A\cap J) + F(A \cup J) - F(J) \leqslant F(A)$ by submodularity. This implies 
that the desired result by considering the representation of the Lov\'asz extension in \eq{lovasz2} and the fact that we have just prove that $\mathcal{P}$ contains the product of the two submodular polyhedra associated to $F^J$ and $F_J$.

(ii) This is immediate from the expression of the Lov\'asz extension in \eq{lovasz2}. Indeed, the order within $J$ and the one within $J^c$ do not interact. Note that this case includes cases where we some of the components of $|w_J|$ are equal to some of $|w_{J^c}|$.

(iii) $\Omega^J$ corresponds to the submodular function obtained as the contraction of $F$ by $J$. It is thus a norm as soon as $F^J$ is positive on all singletons, which is itself equivalent to the stability of $J$. The equivalence of being a norm with stability of the set $J$ is then straightforward.

\subsection{Proof of Proposition~\ref{prop:patterns}}

 Let $Q = \frac{1}{n} X^\top X \in \rb^{p \times p}$ and $r = \frac{1}{n} X^\top y \in \rb^p$.
The unicity of the minimizer $\hat{w}$ is a consequence of the invertibility of $Q = \frac{1}{n} X^\top X$. Let $J \subset V$. We will show that if $\supp(\hat{w}) = J$, then $\hat{w}_J$ is an affine function of $r$ (and hence $y$), of the form $\hat{w}_J = (Q_{JJ}^{-1} - A_{JJ} ) r_J + b_J$, where $ 0 \preccurlyeq A_{JJ} \preccurlyeq Q_{JJ}^{-1}$ and $(A_{JJ},b_J)$ belongs to a finite set independent of $r$. 

If $J$ is not a stable set, then, by Proposition~\ref{prop:subdifferential}, this will implies that there exists $j \in J^c$ such that
$ Q_{j J}  [ (Q_{JJ}^{-1} - A_{JJ} ) r_J + b_J ] - r_{j} ) =0$, i.e., 
$$ 0 = Q_{j J}  [ (Q_{JJ}^{-1} - A_{JJ} ) r_J + b_J ] - r_{j} = \frac{1}{n}[ Q_{j J}  Q_{JJ}^{-1} X_J^\top  - X_{j}^\top  -  Q_{j J}   A_{JJ} X_J^\top ]  y  
+ Q_{j J}  b_J.
$$
The row vector $Q_{j J}  Q_{JJ}^{-1} X_J^\top  - X_{j}^\top  -  Q_{j J}   A_{JJ} X_J^\top$ cannot be equal to zero, otherwise,
$$
0 = \frac{1}{n} [ Q_{j J}  Q_{JJ}^{-1} X_J^\top  - X_{J^c}^\top  -  Q_{j J}   A_{JJ} X_J^\top ] X_j
= Q_{j J}  Q_{JJ}^{-1}Q_{Ji}   - Q_{jj}   -  Q_{j J}   A_{JJ} Q_{Jj}  \leqslant  Q_{j J}  Q_{JJ}^{-1}Q_{Ji}   - Q_{jj} 
$$
which is a contradiction because of the invertibility of $Q$ and the Schur complement lemma~\cite{golub83matrix} (which implies that the previous quantity must be strictly negative). Thus, we have shown that if
$\supp(\hat{w}) = J$ and $J$ is not a stable subset, then for a finite number of non zero $(c,d) \in \rb^n \times \rb$, then $c^\top y$ is constant. This occurs with probability zero.

What remains to be shown is the affine representation of $\hat{w}_J$ when the support is given; it is essentially equivalent to showing that the path is piecewise affine, which is not surprising for a polyhedral norm~\cite{rosset}. We use the representation $\Omega_J(w_J) = \max_{z \in B} z^\top w_J$ where $B$ is the finite set of $z$ such that $|z|$ in an extreme point of the submodular polyhedron associated with $\Omega_J$.

\emph{Necessary optimality conditions}~\cite{borwein2006caa} for such the problem in \eq{objective} is the existence of $\eta_z \geqslant 0 $ (for each $z \in B$) such that (1) $\sum_{z \in B }\eta_z = 1$,   (2) $\eta_z = 0$ if $z$ is not a maximizer
of $\max_{z \in B} z^\top w_J$, and (3) $w_J$ is a minimizer of
$  \frac{1}{2}w_J^\top Q_{JJ} w_J - r_J^\top w_J +  \lambda w^\top \sum_{z \in A } \eta_z z  $, i.e.,
$ Q_{JJ} w_J  +  \lambda   \sum_{z \in A } \eta_z z =r_J
$.  Moreover, by Carath\'eodory's theorem~\cite{borwein2006caa}, the number $k$ of non-zero $\eta$ may be taken to be less than $|J|+1$.

This thus implies that, if consider the vector $\zeta \in \rb^{ k}$ of non-zero $\eta$, and the matrix $Z \in \rb^{|J| \times k}$ of corresponding $z$'s, then we have
$$ Q_{JJ} w_J  +  \lambda  Z \zeta = r_J$$
$$ \zeta^\top 1 = 1 $$
$$ \exists c \in \rb \mbox{ such that }  Z^\top w_J = c 1.$$
In matrix form, this can be written as:
$$
\left( \begin{array}{ccc}
Q_{JJ} & \lambda Z & 0 \\
\lambda Z^\top & 0 & - \lambda 1 \\
0 & - \lambda 1^\top & 0 
\end{array} \right)
\left( \begin{array}{c}
w_J \\ \zeta \\ c
\end{array} \right)
=\left( \begin{array}{c}
 r_J \\ 0 \\ -\lambda
\end{array} \right).
$$
It is then a simple linear algebra exercise to show that if $k \leqslant |J|+1$, then $w_J$ is of the desired form.

\subsection{Proof of Proposition~\ref{prop:support}}

  Let $q = \frac{1}{n} X^\top \varepsilon \in \rb^p$, which is normal with mean zero and covariance matrix $\sigma^2 Q / n$.
We have $\Omega(x) \geqslant \Omega_J(x_J) +  \Omega^{J}(x_{J^c})
\geqslant \Omega_J(x_J) +  \rho(J) \Omega_{J^c}(x_{J^c}) \geqslant \rho(J) \Omega(x)$. This implies that
$\Omega^\ast(q) \geqslant \rho(J)^{-1} \max \{ \Omega_J^\ast(q_J) , (\Omega^J)^\ast(q_{J^c}) \}$. Moreover,
$q_{J^c} - Q_{J^c J}Q_{JJ}^{-1} q_J$ is normal with covariance matrix $\sigma^2 / n ( Q_{J^c J^c} - Q_{J^c J} Q_{JJ}^{-1} Q_{J J^c} ) \preccurlyeq \sigma^2 / n  Q_{J^c J^c} $. This implies that with probability larger than $1 - 3 P( \Omega^\ast(q) > \lambda  \rho(J) \eta/2 )$,
we have  $\Omega_J^\ast(q_J) \leqslant \lambda/2$ and 
 $(\Omega^J)^\ast( q_{J^c} - Q_{J^c J}Q_{JJ}^{-1} q_J ) \leqslant \lambda  \eta/2 $.
 
 We denote by $\tilde{w}$ the unique (because $Q_{JJ}$ is invertible) minimum of  $\frac{1}{2n} \| y - X w\|_2^2 + \lambda \Omega(w)$, subject to $w_{J^c}=0$. $\tilde{w}_J$ is defined through $Q_{JJ} ( \tilde{w}_J - {w_J}^\ast ) - q_J = - \lambda s_J$ where $s_J \in \partial \Omega_J(\tilde{w}_J)$ (which implies that $\Omega_J^\ast(s_J) \leqslant 1$) , i.e., $\tilde{w}_J - w^\ast_J = Q_{JJ}^{-1} ( q_J - \lambda s_J)$.  We have:
 \BEAS
 \| \tilde{w}_J - w^\ast_J  \|_\infty  & \leqslant &  
 \max_{ j \in J } | \delta_j^\top  Q_{JJ}^{-1} ( q_J - \lambda s_J) |
 \\
  & \leqslant &  
 \max_{ j \in J } \Omega_J(   Q_{JJ}^{-1}  \delta_j ) \Omega_J^\ast( q_J - \lambda s_J) |
 \\ 
  & \leqslant &  
 \max_{ j \in J } c(J)  \|   Q_{JJ}^{-1}  \delta_j  \|_2 [ \Omega_J^\ast( q_J) +  \lambda   \Omega_J^\ast( s_J)  ]
\leqslant 
\frac{3}{2} \lambda c(J)   \kappa^{-1} 
 \\ 
 \EEAS
 Thus if $2 \lambda c(J)   \kappa^{-1}  \leqslant \nu$, then $\supp(\tilde{w}) \supset \supp(w^\ast)$.
 
 We now show that since we have $(\Omega^J)^\ast( q_{J^c} - Q_{J^c J}Q_{JJ}^{-1} q_J ) \leqslant \lambda  \eta/2 $, $\tilde{w}$ is the unique minimizer of \eq{objective}. For that it suffices to show that $(\Omega^J)^\ast ( Q_{J^c J} (\tilde{w}_J - w_J^\ast) - q_{J^c} ) < \lambda$. We have:
 \BEAS
 (\Omega^J)^\ast ( Q_{J^c J} (\tilde{w}_J - w_J^\ast) - q_{J^c} )
 & = &  (\Omega^J)^\ast ( Q_{J^c J} Q_{JJ}^{-1} ( q_J - \lambda s_J)- q_{J^c} )
\\
 & \leqslant &  (\Omega^J)^\ast ( Q_{J^c J} Q_{JJ}^{-1}   q_J  - q_{J^c} ) + 
\lambda  (\Omega^J)^\ast (  Q_{J^c J} Q_{JJ}^{-1}  s_J )
\\
 & \leqslant &  (\Omega^J)^\ast ( Q_{J^c J} Q_{JJ}^{-1}   q_J  - q_{J^c} ) + 
\lambda (\Omega^J)^\ast[  ( \Omega_J(   Q_{JJ}^{-1} Q_{Jj} ) )_{j \in J^c} ]
\\
& \leqslant &  \lambda  \eta/2 + \lambda ( 1-\eta) < \lambda
 \EEAS
 which leads to the desired result.

\subsection{Proof of Proposition~\ref{prop:high-dim}}

   Like for the proof of Proposition~\ref{prop:support}, we have $\Omega(x) \geqslant \Omega_J(x_J) +  \Omega^{J}(x_{J^c})
\geqslant \Omega_J(x_J) +  \rho(J) \Omega_{J^c}(x_{J^c}) \geqslant \rho(J) \Omega(x)$. Thus, if we assume
$\Omega^\ast(q) \leqslant  \lambda  \rho(J) /2$, then  $\Omega_J^\ast(q_J) \leqslant \lambda/2$ and
 $(\Omega^J)^\ast( q_{J^c}) \leqslant \lambda /2 $. Let $\Delta  = \hat{w} - w^\ast$.

We follow the proof from~\cite{bickel_lasso_dantzig} by using the decomposition property of the norm $\Omega$.
We have, by optimality of $\hat{w}$:
$$ \frac{1}{2}\Delta^\top Q \Delta+   \lambda \Omega( w^\ast + \Delta) + q^\top \Delta \leqslant \lambda \Omega( w^\ast + \Delta) + q^\top \Delta  \leqslant    \lambda \Omega( w^\ast)
$$
Using the decomposition property,
$$  \lambda \Omega_J( (w^\ast + \Delta)_J ) + \lambda \Omega^J((w^\ast + \Delta)_{J^c} ) + q_J^\top \Delta_J
+ q_{J^c}^\top \Delta_{J^c}
  \leqslant    \lambda \Omega_J( w^\ast_J)
$$
$$   \lambda \Omega^J( \Delta_{J^c} )  \leqslant    \lambda \Omega_J( w^\ast_J) - \lambda \Omega_J( w^\ast_J + \Delta_J ) 
+ \Omega_J^\ast(q_J) \Omega_J(\Delta_J)
+ 
(\Omega^J)^\ast (q_{J^c}) \Omega^J(  \Delta_{J^c})
$$
$$   ( \lambda - (\Omega^J)^\ast (q_{J^c}) ) \Omega^J( \Delta_{J^c} )  \leqslant    ( \lambda +  \Omega_J^\ast(q_J)   ) \Omega_J( \Delta_J )$$
Thus $\Omega^J( \Delta_{J^c} )  \leqslant  3 \Omega_J( \Delta_J)$, which implies $
\Delta^\top Q \Delta \geqslant \kappa \|\Delta_J \|_2^2
$ (we have assumed a restricted eigenvalue condition).
Moreover, we have:
\BEAS
\Delta^\top Q \Delta & = & \Delta^\top ( Q \Delta) \leqslant \Omega(\Delta) \Omega^\ast( Q \Delta) \\
& \leqslant &  \Omega(\Delta) ( \Omega^\ast( q)  +  \lambda ) \leqslant \frac{3 \lambda}{2} \Omega(\Delta) 
\mbox{ by optimality of } \hat{w}
\\
\Omega(\Delta)
& \leqslant & 
   \Omega_J(\Delta_J) + \rho(J)^{-1} 
\Omega^J(\Delta_{J^c})  \\
& \leqslant &  \Omega_J(\Delta_J) ( 3 + \frac{1}{\rho(J)} )
\leqslant  \frac{4}{\rho(J)} \Omega_J(\Delta_J)
\EEAS
This implies that $\frac{ \kappa}{c(J)^2} \Omega_J(\Delta_J)^2 \leqslant \kappa \| \Delta_J\|_2^2 \leqslant \Delta^\top Q \Delta \leqslant 
\frac{6 \lambda}{\rho(J)} \Omega_J(\Delta_J)$, and thus 
$
\Omega_J(\Delta_J) \leqslant \frac{6 c(J)^2 \lambda}{\kappa \rho(J)}
$, which leads to the desired result, given the previous inequalities.

\subsection{Proof of Proposition~\ref{prop:proba}}
 
 We have $\Omega^\ast(z) = \max_{ \Omega(w) \leqslant 1 } w^\top z$; the maximum can be taken over the set of extreme points of the unit ball, which leads to the desired result given Proposition~\ref{prop:extremepoints}.

\bibliographystyle{unsrt}
\bibliography{submodular}

\begin{thebibliography}{10}

\bibitem{cap}
P.~Zhao, G.~Rocha, and B.~Yu.
\newblock Grouped and hierarchical model selection through composite absolute
  penalties.
\newblock {\em Annals of Statistics}, 37(6A):3468--3497, 2009.

\bibitem{jenatton2009structured}
R.~Jenatton, J.Y. Audibert, and F.~Bach.
\newblock Structured variable selection with sparsity-inducing norms.
\newblock Technical report, arXiv:0904.3523, 2009.

\bibitem{huang2009learning}
J.~Huang, T.~Zhang, and D.~Metaxas.
\newblock {Learning with structured sparsity}.
\newblock In {\em Proc. ICML}, 2009.

\bibitem{LaurentGuillaumeGroupLasso}
L.~Jacob, G.~Obozinski, and J.-P. Vert.
\newblock Group {L}asso with overlaps and graph {L}asso.
\newblock In {\em Proc. ICML}, 2009.

\bibitem{kim}
S.~Kim and E.~Xing.
\newblock Tree-guided group {L}asso for multi-task regression with structured
  sparsity.
\newblock In {\em Proc. ICML}, 2010.

\bibitem{jenattonmairal}
R.~Jenatton, J.~Mairal, G.~Obozinski, and F.~Bach.
\newblock Proximal methods for sparse hierarchical dictionary learning.
\newblock In {\em Proc. ICML}, 2010.

\bibitem{Mairal10aNIPS}
J.~Mairal, R.~Jenatton, G.~Obozinski, and F.~Bach.
\newblock Network flow algorithms for structured sparsity.
\newblock In {\em Adv. NIPS}, 2010.

\bibitem{haupt2006signal}
J.~Haupt and R.~Nowak.
\newblock Signal reconstruction from noisy random projections.
\newblock {\em IEEE Transactions on Information Theory}, 52(9):4036--4048,
  2006.

\bibitem{submodular_tutorial}
{F} {B}ach.
\newblock Convex analysis and optimization with submodular functions: a
  tutorial.
\newblock Technical Report 00527714, HAL, 2010.

\bibitem{krause2005near}
A.~Krause and C.~Guestrin.
\newblock Near-optimal nonmyopic value of information in graphical models.
\newblock In {\em Proc. UAI}, 2005.

\bibitem{kawahara22submodularity}
Y.~Kawahara, K.~Nagano, K.~Tsuda, and J.A. Bilmes.
\newblock Submodularity cuts and applications.
\newblock In {\em Adv. NIPS}, 2009.

\bibitem{fujishige2005submodular}
S.~Fujishige.
\newblock {\em Submodular Functions and Optimization}.
\newblock Elsevier, 2005.

\bibitem{edmonds}
J.~Edmonds.
\newblock Submodular functions, matroids, and certain polyhedra.
\newblock In {\em Combinatorial optimization - Eureka, you shrink!}, pages
  11--26. Springer, 2003.

\bibitem{negahban2008joint}
S.~Negahban and M.~J. Wainwright.
\newblock Joint support recovery under high-dimensional scaling: Benefits and
  perils of $\ell_1$-$\ell_\infty$-regularization.
\newblock In {\em Adv. NIPS}, 2008.

\bibitem{lovasz1982submodular}
L.~Lov{\'a}sz.
\newblock Submodular functions and convexity.
\newblock {\em Mathematical programming: the state of the art, Bonn}, pages
  235--257, 1982.

\bibitem{choquet1953theory}
G.~Choquet.
\newblock Theory of capacities.
\newblock {\em Ann. Inst. Fourier}, 5:131--295, 1954.

\bibitem{orlin2009faster}
J.B. Orlin.
\newblock A faster strongly polynomial time algorithm for submodular function
  minimization.
\newblock {\em Mathematical Programming}, 118(2):237--251, 2009.

\bibitem{boyd}
S.~P. Boyd and L.~Vandenberghe.
\newblock {\em Convex Optimization}.
\newblock Cambridge University Press, 2004.

\bibitem{SparseStructuredPCA}
R.~Jenatton, G.~Obozinski, and F.~Bach.
\newblock Structured sparse principal component analysis.
\newblock In {\em Proc. AISTATS}, 2009.

\bibitem{tibshirani2005sparsity}
R.~Tibshirani, M.~Saunders, S.~Rosset, J.~Zhu, and K.~Knight.
\newblock Sparsity and smoothness via the fused {L}asso.
\newblock {\em J. Roy. Stat. Soc. B}, 67(1):91--108, 2005.

\bibitem{horn1990matrix}
R.~A. Horn and C.~R. Johnson.
\newblock {\em Matrix analysis}.
\newblock Cambridge Univ. Press, 1990.

\bibitem{ando1979concavity}
T.~Ando.
\newblock Concavity of certain maps on positive definite matrices and
  applications to hadamard products.
\newblock {\em Linear Algebra and its Applications}, 26:203--241, 1979.

\bibitem{mallows}
C.~L. Mallows.
\newblock Some comments on {$C_p$}.
\newblock {\em Technometrics}, 15(4):661--675, 1973.

\bibitem{wipf}
D.~Wipf and S.~Nagarajan.
\newblock Sparse estimation using general likelihoods and non-factorial priors.
\newblock In {\em Adv. NIPS}, 2009.

\bibitem{beck2009fast}
A.~Beck and M.~Teboulle.
\newblock A fast iterative shrinkage-thresholding algorithm for linear inverse
  problems.
\newblock {\em SIAM Journal on Imaging Sciences}, 2(1):183--202, 2009.

\bibitem{chambolle2009total}
A.~Chambolle and J.~Darbon.
\newblock {On total variation minimization and surface evolution using
  parametric maximum flows}.
\newblock {\em International Journal of Computer Vision}, 84(3):288--307, 2009.

\bibitem{negahban2009unified}
S.~Negahban, P.~Ravikumar, M.~J. Wainwright, and B.~Yu.
\newblock {A unified framework for high-dimensional analysis of M-estimators
  with decomposable regularizers}.
\newblock In {\em Adv. NIPS}, 2009.

\bibitem{Zhaoyu}
P.~Zhao and B.~Yu.
\newblock On model selection consistency of {L}asso.
\newblock {\em Journal of Machine Learning Research}, 7:2541--2563, 2006.

\bibitem{cevher}
A.~Krause and V.~Cevher.
\newblock Submodular dictionary selection for sparse representation.
\newblock In {\em Proc. ICML}, 2010.

\bibitem{golub83matrix}
G.~H. Golub and C.~F.~Van Loan.
\newblock {\em Matrix Computations}.
\newblock Johns Hopkins University Press, 1996.

\bibitem{rosset}
S.~Rosset and J.~Zhu.
\newblock Piecewise linear regularized solution paths.
\newblock {\em Ann. Statist.}, 35(3):1012--1030, 2007.

\bibitem{borwein2006caa}
J.~M. Borwein and A.~S. Lewis.
\newblock {\em Convex Analysis and Nonlinear Optimization: Theory and
  Examples}.
\newblock Springer, 2006.

\bibitem{bickel_lasso_dantzig}
P.~Bickel, Y.~Ritov, and A.~Tsybakov.
\newblock {Simultaneous analysis of Lasso and Dantzig selector}.
\newblock {\em Annals of Statistics}, 37(4):1705--1732, 2009.

\end{thebibliography}

\end{document}